\documentclass[a4paper,twoside,12pt]{article}
\usepackage[english]{babel}

\usepackage{amsmath,amsfonts,amssymb,amsthm}
\newtheorem{theorem}{Theorem}
\newtheorem{lemma}{Lemma}

 \title{  {\bf  Diagonal  Multi-omics  Integration  of Heterogeneous Datasets} }

\author{ Maksim \,V.~Kukushkin,  Mikhail S. Arbatskiy,  Dmitriy E. Balandin,\\ Alexey V. Churov   \\ \\
  \small  \textit{Pirogov Russian National Research Medical University,}\\
    \small  \textit{Russian Clinical and Research Center of Gerontology,}\\
    \small  \textit{ Ministry of Health of the Russian Federation,  129226, Moscow, Russia}\\
  \small\textit{Russian Academy of Sciences,  Kabardino-Balkarian Scientific Center,}\\
 \small\textit{Institute of Applied Mathematics and Automation, 360000,  Nalchik, Russia}\\}

\date{}
\begin{document}
\maketitle

\begin{abstract}

 In this paper, we consider methods for the diagonal multi-omics integration of heterogeneous datasets. Several approaches to the nature of biological heterogeneity are analyzed and developed to comprehend  more clearly the generated differences. Specifically, the extremal trace problems for the coupled Laplacian on sets homeomorphic to the Stiefel manifold embedded in the complex Euclidean space are investigated. The gradient ascent method for the maximization problem is elaborated in the classical terms of functional analysis, which is of significant interest in itself. On this basis, we introduce a novel characteristic of dataset heterogeneity by employing  the norm of the difference between the maximum and minimum points.

\end{abstract}
\begin{small}\textbf{Keywords:}
  Manifold alignment; Unsupervised topological alignment
of single-cell multi-omics integration; Graph Laplacian;     Single cell RNA-seq; Single cell epigenomics.   
\end{small}

\section{Introduction}

Current biomedical research is experiencing a period of massive accumulation of molecular data. High-throughput technologies enable the simultaneous measurement of thousands of molecular features across a fixed set of objects at multiple levels of cellular organization, including genomic, transcriptomic, proteomic, and metabolomic profiles. The primary objective of integrative analysis is to establish meaningful structural relationships among these diverse feature sets to uncover the underlying mechanisms of complex diseases, identify biomarkers, and map molecular regulatory programs. However, despite rapid technological progress, the mathematical formalization of these heterogeneous feature spaces remains a fundamental challenge. A living cell behaves  as a complex dynamic system where distinct feature classes are interconnected in a non-linear manner, governed by feedback loops and varying time scales. Attempting to describe such systems exclusively through classical statistical correlation and   linear projection methods, such as those implemented in multivariate exploration packages like mixOmics \cite{Rohart2017}, often proves insufficient. Such approaches typically identify jointly varying features without capturing the structural and directional dependencies inherent to the system.

 To provide unified representations, several probabilistic and machine learning models have been developed \cite{Zhang2022,Ashuach2023,KaiCao2022}.   While these methods offer valuable computational tools, they are bound by critical methodological limitations.   A large number  of current alignment algorithms require  artificial penalty functions or heavily rely  on the strict structural equivalence of the underlying data matrices, limiting their mathematical generality.  Most methods reduce data to low-dimensional representations optimized solely for visualization or clustering, failing to provide an analytical measure of the structural consistency between the coupled feature systems.

The aforementioned limitations underscore the need to transition from heuristic matrix alignment to a fundamental operator theory. Instead of imposing artificial constraints on the data structure, a more natural mathematical approach involves the direct coupling of data matrices independent of any   regularizing constructions. From a practical perspective,  our research group has previously investigated various aspects of this problem, including genomic mutations \cite{Mironenko2022}, single-cell transcriptional profiles \cite{Arbatskiy2024}, proteomic data \cite{Kulebyakina2024}, and small non-coding RNAs \cite{Basalova2020}.

Extensive research has focused on advancing the  concept of embedding heterogeneous measurements into a unified latent space \cite{Joshua2017,JieLiu2019,Duren2018}. For instance, \cite{KaiCao2020} introduces a new method for unsupervised topological alignment
of  multi-omics data, which  does not require prior information on the correspondence between cells or measurements. Similarly, the method in \cite{Dou2020} establishes a topological alignment between two distinct modalities using a transition matrix to project one feature space onto another. However,   the approaches devoted to coupling the heterogeneous datasets and their mapping into a common latent space should  be classified and systematized in order to identify a generalized  mathematical structure that is invariant with respect to several classical methods of   coupled manifold alignment.

In the paper \cite{KABaChu26}, we elaborate the coupled graph Laplacian method \cite{KABaChu25},\cite{Belkin2003} for the alignment of topological structures in Euclidean spaces, the central idea of which is to create a mapping of various sets onto a space in which images can be comparable in some sense. The problem can be considered in the context of the analysis of heterogeneous data, regardless of any assumptions on the coupled matrix corresponding to the  datasets. Moreover, we provide a qualitative theory of the extremal trace problem on   sets homeomorphic to the Stiefel manifold for the coupled graph Laplacian operator, which is of interest in itself. Generally, the main scientific gap that is   eliminated is the absence of a unified feature space for analyzing data of different natures  lying on various low-dimensional manifolds in their corresponding high-dimensional initial spaces \cite{Ham2003}.

In this paper, inspired by a complete description of the extremal trace problem  given in \cite{KABaChu26},  we  study   the questions related to the biological application of the extremal points on the sets homeomorphic to the Stiefel manifold. The central focus  is the deployment of the maximization approach that  has not been considered previously  in the context  of    manifold  alignment    methods.

 The minimum problem lies at the foundation of the conceptual idea of the topological alignment. Its solution represents images of the datasets in the latent space where the topological neighborhoods of the points are positioned  as close as possible. The latter phenomenon is described by the term {\it alignment}.
The direct consequence for the resulting low-dimensional embedding is a total
flattening of the geometric structure. In the latent space generated by the minimum
trace problem, tightly interconnected objects are continuously condensed  into a  highly smoothed and low-dimensional continuous manifold.
While this global alignment exposes the fundamental harmonious  structure shared by the
independent feature systems, it inherently introduces an oversmoothing effect,
obscuring subtle localized variations and fine-grained subpopulation boundaries. This tendency to suppress significant localized phenomena creates  a distinct methodological limitation, thereby establishing a direct theoretical prerequisite for introducing the maximization approach, which is specifically elaborated  to capture and strengthen underlying structural heterogeneities.

Conversely, the maximization approach identifies the directions of extreme structural discrepancy generated by phase shifts between the real and imaginary parts of the coupled Laplacian. Whereas the consensus model, corresponding to the minimum problem, flattens local gradients, the maximum problem strengthens them, thereby acting as a high-contrast mathematical tool to uncover hidden stratification.
 This mathematical concept introduces a fundamental duality between two optimization problems.

The fundamental limitation of traditional manifold alignment lies in the inherent
relativity and isolated blindness of the minimization approach. By design, the minimum problem
provides a smoothed projection of the data; however, this invariant structure remains
methodologically blind, as it cannot analytically determine whether it has suppressed
unimportant stochastic noise or erased critical regulatory anomalies. Evaluating
the embedding against an external canonical standard, such as the   real or imaginary
components of the graph  Laplacian considered  independently, breaks the algebraic symmetry of the
heterogeneous feature spaces by arbitrarily prioritizing one dataset over another.

To resolve this uncertainty, we propose employing   the operator norm of the difference between the maximum and minimum points  as an intrinsic  mathematical
standard. Rather than relying on external regularizers, this indicator evaluates
the maximum potential divergence between the   smoothing   provided by the minimization approach
and the   separation induced by the ``microscope'' of the maximization
approach. Even in scenarios where   the embedding corresponding to the maximum point  is not the primary
object of visualization, the norm of the difference serves as a quantitative measure
of the alignment stability. It indicates the exact value
by which the underlying structural conflicts can distort the latent space. Without
this  indicator, any  mapping  remains fundamentally incomplete and speculative,
whereas its deployment  delivers a mathematically rigorous, verified evaluation of
the true geometric nature of the coupled spaces.

\section{Preliminaries}

Throughout the paper, we consider   finite dimensional matrices  $\mathbb{C}^{n\times m},\;n,m\in \mathbb{N}$ over the field of  complex numbers  and use the following notation $
A =\{a_{sj}\}\in \mathbb{C}^{n\times m}, \,s=1,2,\dots,n,\;j=1,2,\dots,m
$  for a matrix. We use the following standard notation for the   transpose operation
$
A^{T}=\{a_{js}\}.
$
 Consider a matrix
$
A= \{a_{sj}\}\in \mathbb{C} ^{n\times  m},\;a_{sj}\in \mathbb{C},
$
denote
$
\mathbf{a}^{\cdot}_{s }:=
 ( a_{s1}, a_{s2},\dots,a_{sm})^{T},\,\mathbf{a}_{ j }: = (a_{1j},$ $a_{2j},\dots,a_{nj})^{ T }.
$
We consider a complex linear vector space  $\mathbb{C}^{n}$ consisting of the set of column matrices  whose elements are complex numbers $\mathbf{x}=(x_{1},x_{2},\dots,x_{n})^{T},\; x_{j} \in \mathbb{C},\,j=1,2,\dots,n.$
Analogously, we define the real  linear vector space  $\mathbb{R}^{n}.$
A complex linear vector space equipped with the inner (scalar) product structure defined below is called a complex Euclidean space
$
(\mathbf{x},\mathbf{y})_{\mathbb{C}^{n}}:=\sum_{j=1}^{n}x_{j}\bar{y}_{j},\;\mathbf{x},\mathbf{y}\in \mathbb{C}^{ n}.
$
   Define the norm in the general sense  in a space with the inner  product as follows
$
\|\mathbf{x}\|=\sqrt{(\mathbf{x},\mathbf{x})}.
$
Consider a   matrix $A =\{a_{sj}\}\in \mathbb{C}^{n\times m},$ denote $\bar{A} =\{\bar{a}_{sj}\}.$ We use the following notations for the Hermitian components of the operator
$
 \mathfrak{Re}A=(A+A^{\ast})/2,\;\mathfrak{Im}A=(A-A^{\ast})/2i,\;\bar{A}:=\{\bar{a}_{sj} \},\;\mathrm{Re}A:=\{\mathrm{Re}a_{sj} \},\;\mathrm{Im}A:=\{\mathrm{Im}a_{sj} \},
$
 the latter matrices are called  the real and imaginary part of the matrix $A$ respectively.
The detailed information on the Hermitian components properties is given in the paper \cite{Math2024}. Denote by    $    \mathrm{D}   (A),\,   \mathrm{R}   (A),\,\mathrm{N}(A)$      the    domain of definition, the  range,  and the  kernel or   null space  of the  operator $A$ respectively.

  Consider a set of   elements
 $
 \{\mathbf{x}_{1},\mathbf{x}_{2},\dots,\mathbf{x}_{n} \}
 $
 belonging to  a normed space. Generally, in terms of applications, each vector $\mathbf{x}_{j}$ represents  measurements  corresponding to a concrete object. In particular, if we consider the  problem of biological integration, we assume that $\mathbf{x}_{j}\in \mathbb{R}^{q},\,q\in \mathbb{N},$ where each coordinate reflects a concrete measurement.
Let us construct a weighted graph $G$ that has   $n$ vertices, where the vertices correspond to elements, and the  edges represent connections between them.
We index the vertices according to the given order, i.e., the $j$-th vertex corresponds to $\mathbf{x}_{j}.$  We suppose that   the  $s$-th  and $j$-th  vertices   are adjacent  if the elements   $\mathbf{x}_{s}$ and  $\mathbf{x}_{j}$ are related,  for instance, close in some sense. The main challenge   consists in defining  the rather vague notion of  closeness in a more concrete way \cite{Belkin2003}.

Here, we consider   the  $\varepsilon$ - neighborhood method, which postulates the following sense of closeness.  According to this method, the vertices    $s$ and $j$ are adjacent if
 $
 \|\mathbf{x}_{s}-\mathbf{x}_{j} \|   <\varepsilon,
 $
where the norm is understood in the abstract sense. However, the Euclidean norm is typically    used  in applications.
Note that this type of closeness is geometrically motivated, since
the relationship is naturally symmetric. However, it often
leads to a graph  with several connected components, and it is  difficult
to choose a suitable value of $\varepsilon$ to avoid a disconnected graph.

  Having constructed the graph according to the above,  we can assign weights to form the weight matrix. The most important one  relates to the heat kernel. Here, we assume that if   the  $s$-th and $j$-th vertices are connected, then the weight matrix is formed from the elements
$
A_{sj}=e^{- t^{-1}\|\mathbf{x}_{s}-\mathbf{x}_{j}\|^{2} },\;t\in \mathbb{R} \setminus \{0\},\;s,j=1,2,\dots,n\,,
$
otherwise,  $A_{sj}  = 0.$
In accordance with this definition, we obtain the adjacency matrix of the weighted graph
$
A= \{A_{sj}\}\in \mathbb{R}^{n\times n}.
$
 The justification for  choosing this weight is presented in  \cite{Ham2005}, \cite{Ek2009}.

  Consider  heterogeneous datasets
$
 \{\mathbf{x}^{(1)}_{1},\mathbf{x}^{(1)}_{2},\dots,\mathbf{x}^{(1)}_{n} \},\;\{\mathbf{x}^{(2)}_{1},\mathbf{x}^{(2)}_{2},\dots,\mathbf{x}^{(2)}_{n} \},
$
where the elements within a dataset relate to each other.
 In accordance with the above, we can construct two adjacency matrices $A^{(1)},\,A^{(2)}$ corresponding to the datasets.
The main idea of the method introduced in \cite{KABaChu25}, \cite{KABaChu26}  is to  construct a coupling structure for  heterogeneous datasets and apply an  operation that implements a topological alignment  to the datasets coupled with the structure. For this purpose, the graph Laplacian method is generalized in \cite{KABaChu26}.

 Motivated by the   constructions given above,  throughout the paper,  we consider a   symmetric   matrix $W\in \mathbb{C}^{ n\times n },\,\mathrm{Re}W_{sj},\,\mathrm{Im}W_{sj}>0,$  where from the applied point of view,  we can put in correspondence  $\mathrm{Re}W=A^{(1)},\,\mathrm{Im}W=A^{(2)}.$  Here, the coupled structure mentioned above is represented in its  most general form being an algebraic  structure -- the field of  complex numbers. Define a diagonal matrix
$$
D=\{D_{sj}\},\;D_{sj}=\left\{ \begin{aligned}
  \sum\limits_{k=1}^{n}W_{kj},\;s=j \\
 0,\;s\neq j   \\
\end{aligned}
 \right.  .
$$
Further, we assume that $D_{jj}\neq 0,\;j=1,2,\dots,n,$ which does not restrict the appropriate  class of matrices due to the given  sense of the elements.
 Consider an operator  $L:=D-W.$ It is clear that $\mathrm{Re}L,\mathrm{Im}L$ are the graph Laplacian operators.  Although the operator   $L$ is not selfadjoint, the corresponding matrix is symmetric, i.e., $L^{\ast}=\bar{L}.$ The operator $L$ is called   the coupled graph Laplacian operator. Note   that   the kernel of the coupled  graph Laplacian is a one-dimensional subspace of $\mathbb{C}^{n}$ (see Lemma 1, \cite{KABaChu25}) defined as follows
$
\mathrm{N}(L)= \mathrm{span}\{\mathbf{1}\},\,\mathbf{1}=(1,1,\ldots,1).
$
Therefore, taking into account the fact $\mathrm{N}(L)=\mathrm{N}(L^{\ast}),$ we obtain
$
\mathbb{C}^{n}=\mathrm{N}(L) \oplus\mathrm{R}(L).
$
  Since $\mathrm{N}(L)$ is an invariant subspace of the operator  $L,$ it follows that    $\mathrm{R}(L)$  is the invariant subspace of the operator $L$ of  dimension $n-1.$ Therefore,  to avoid the trivial solution of the  optimization problem involving the operator $L,$   we should  consider the subspace $\mathrm{R}(L).$
Using the symmetric property of the matrix $W$ (see \cite{KABaChu25}) it is not hard to prove that
 \begin{equation}\label{1}
  \mathrm{tr}\{X^{\ast}LX\} =  \sum\limits_{j=1}^{m}       ( L\mathbf{x}   _{ j },       \mathbf{x}   _{ j })_{\mathbb{C}^{n}}  =\frac{1}{2} \sum\limits_{s,j=1}^{n}\|\mathbf{x}^{\cdot}_{s}-\mathbf{x}^{\cdot}_{j}\|_{\mathbb{C}^{m}}^{2}W_{sj}.
\end{equation}
  Consider the so-called connection condition
\begin{equation}\label{2}
X^{\ast}T_{\zeta}X= I_{m},\;X\in \mathbb{C}^{n\times m},\,m<n,
\end{equation}
where
$
T_{\zeta}:=\mathfrak{Re}\left( D\, \mathrm{diag}\{e^{i\zeta_{1}}, e^{i\zeta_{2}},\dots,e^{i\zeta_{n}}\} \right), \,\zeta_{j}  \in (0,\pi/2);
$
in this regard see Remark 1  \cite{KABaChu26}.
Define the sets
$
\mathrm{R}^{m}\!(L):=\{X\in \mathbb{C}^{n\times m}:  \mathbf{x}_{j}\in\mathrm{R}(L),\,j=1,2,\dots,m \},$  $\mathfrak{M}^{n}_{m}:=\{ X\in \mathrm{R}^{m}(L):\,X^{\ast}T_{\zeta}X=I_{m}\}.
$
It is not hard to prove that   $\mathfrak{M}^{n}_{m}$ is a compact connected   manifold without boundary \cite{KABaChu26}.
For the convenience, we use the following notations
$$
\psi(X):=\mathrm{tr}(X^{\ast}LX),  \,H(X):= \overline{\psi (X)}L+\psi(X) L^{\ast},\,\Psi(\theta)=e^{-i\theta}L+e^{i\theta}L^{\ast},\;\theta\in \mathbb{R}.
$$
\begin{lemma}\label{L1}The following equality holds:
$
\mathrm{N}(\Psi(\theta))=\mathrm{N}(L),\;\theta\in (0,\pi/2).
$
\end{lemma}
\begin{proof}
Consider     an operator $e^{-i\theta}L.$ The  numerical range  can be easily established due to the following reasoning:
$
(e^{-i\theta}L\mathbf{x},\mathbf{x})_{\mathbb{C}^{n}}=e^{-i\theta}(L\mathbf{x},\mathbf{x})_{\mathbb{C}^{n}},\Rightarrow \Theta(e^{-i\theta}L)=e^{-i\theta}\Theta(L).
$
On the other hand, assume that $m=1$ in formula \eqref{1}, then we have
$$
        ( L\mathbf{x} ,       \mathbf{x} )_{\mathbb{C}^{n}} =\frac{1}{2} \sum\limits_{s,j=1}^{n}|x_{s}-x_{j}| ^{2}W_{sj},
$$
where $\mathrm{arg}\,W_{sj}\in (0,\pi/2).$
Therefore $\mathrm{arg }\,  \Theta(L)\subset (0,\pi/2).$
Since $\theta\in (0,\pi/2),$ it follows that     $\mathrm{arg}\, \Theta(e^{-i\theta}L)\subset (-\pi/2,\pi/2).$ Taking into account
$
(\Psi(\theta)\mathbf{x},\mathbf{x})_{\mathbb{C}^{n}}=2\mathrm{Re} (e^{-i\theta}L\mathbf{x},\mathbf{x})_{\mathbb{C}^{n}},
$
we conclude that
$
 (\Psi(\theta)\mathbf{x},\mathbf{x})_{\mathbb{C}^{n}}>0,\, \mathbf{x}\notin \mathrm{N}(L).
$
Furthermore, using the symmetric property of the operator $L$ and the fact  $L^{\ast}=\bar{L}$, we obtain   $\mathrm{N}(L)=\mathrm{N}(L^{\ast}).$
This yields the desired result.
\end{proof}

\section{The  biological meaning of   extremal trace problems}Consider the following  problems
 \begin{equation}\label{3}
  \mathrm{PI})\;\;\; \mathfrak{A}_{m} :=\mathrm{arg}\,\min\limits_{X\in \mathfrak{M}^{n}_{ m}  } \left|\mathrm{tr}\{X^{\ast}LX\}\right|,\;\; \mathrm{PII})\;\;\;
  \mathfrak{B}_{m} :=\mathrm{arg}\,\max\limits_{X\in \mathfrak{M}^{n}_{ m}  } \left|\mathrm{tr}\{X^{\ast}LX\}\right|.
 \end{equation}
 It is remarkable that  from the applied point of view, the main concept of PI (PII)  is to select some compact manifold  to minimize or maximize the absolute value of   the functional \eqref{3} reflecting  how the images of the elements   $W$ relate to each other under the    mapping constructed by    $\mathfrak{A}_{m}$ or $\mathfrak{B}_{m}$ correspondence. In this regard,     the choice of the compact  manifold is mostly dictated by the technical peculiarities. More precisely, we have a concrete  correspondence satisfying   biologically determined   conditions revealed  in mathematical construction  \eqref{3}, i.e.,
$
\mathfrak{A}^{T}_{m}:\mathbb{C}^{n\times n}\rightarrow \mathbb{C}^{m\times n},$ $\mathfrak{B}^{T}_{m}:\mathbb{C}^{n\times n}\rightarrow \mathbb{C}^{m\times n}.
$
This correspondence allows us to reduce the dimension of the initial high-dimensional space by mapping the data into the compressed transposed representations, taking into account their mutual nature  reflected by the involved complex structure. Consequently, the codomain of the operators corresponds to the created low-dimensional feature space, where  $m$  explicitly denotes the target reduced dimensionality. A matrix belonging to the set $\mathbb{C}^{m\times n}$ is uniquely decomposed into its real and imaginary parts, the columns of which represent the low-dimensional images of the high-dimensional elements  of the initial datasets.

The minimum problem \textnormal{PI} appeals to  the classical heuristic principle of  manifold alignment.  Mathematically, the  minimization of the underlying quadratic form suppresses the total variation
of the coordinates and stabilizes local gradients with respect to the adjacent vertices of the graph. Consequently, if a strong structural coupling exists
between the $s$-th and $j$-th objects that  manifests itself  as a large absolute value  of the complex weight $W_{sj},$ the minimization
of the objective functional ensures that the latent coordinates of the objects are close to each other in the embedding space. In biological terms, the minimum acts as a ``macroscope''. It totally smooths out individual regulatory fluctuations and anomalies to preserve global continuity and reveal the invariant structure of the macro-biological background. However, this total flattening inevitably suffers from an oversmoothing effect,
eliminating critical micro-variations within highly localized subsets of objects.

Conversely, the maximum problem \textnormal{PII} exhibits a strictly opposite structure. Rather than smoothing the geometric structure, it  identifies directions of extreme trigonometric tension   generated by phase shifts between the real and imaginary parts of the coupled graph Laplacian operator.
 Whereas the minimization  approach   attempts to  flatten  the data, the maximum  \textnormal{PII} problem explodes  the local structure.  If a  regulatory  desynchronization (such as an abnormal translational lag between RNA and protein) exists between objects that are otherwise close in the initial feature  space, the maximization  approach increases  the Euclidean distance between their corresponding latent coordinates.  Thus, the maximum problem  \textnormal{PII}   serves as a high-contrast ``microscope'' that reveals  unseen  sample stratification and uncovers   latent cryptic components, presumably represented by cell fractions with hidden drug resistance.

To numerically evaluate this structural conflict on the manifold, we propose a special indicator defined as the operator  norm of the difference between the optimal embedding operators
\begin{equation}\label{4}
\Delta_ m  = \|\mathfrak{A}^{T}_{m} - \mathfrak{B}^{T}_{m}\|_{\mathbb{C}^{n}\rightarrow \mathbb{C}^{m}}.
\end{equation}
 The indicator \eqref{4} rigorously measures the maximum geometric divergence between the linear subspaces spanned by the columns of the transposed matrices representing the solutions to the extremal problems, in the sense of   low-dimensional complex Euclidean space $\mathbb{C}^{m}.$
Employing the operator norm enables a straightforward upper bound via the Frobenius norm, which guarantees the theoretical stability of the indicator boundaries on the compact manifold $\mathfrak{M}^{n}_{m}.$

\vspace{2mm}
\noindent
\begin{minipage}[t]{0.46\textwidth}
\vspace{0pt}
The developed mathematical concept  operates sequentially, passing from raw data matrices to the evaluation of structural conflicts. As illustrated in the diagram, the initial heterogeneous multi-omics profiles are mapped through the coupled graph Laplacian and then  split via the corresponding extremal problems. The metric gap directly evaluates the divergence between these outputs, thereby  distinguishing stable biological systems from hidden stratified populations.
\end{minipage}
\hfill
\begin{minipage}[t]{0.50\textwidth}
\vspace{0pt}
\centering
\setlength{\unitlength}{0.54mm} 
\begin{picture}(95,104)
    \put(20,88){\framebox(54,14)[c]{\scriptsize \textbf{\shortstack{Multi-omics\\Data}}}}
    \put(47,88){\vector(0,-1){5}}

    \put(20,69){\framebox(54,14)[c]{\scriptsize \textbf{\shortstack{Coupled\\Laplacian $L$}}}}

    \put(20,76){\line(-1,0){5}}
    \put(15,76){\vector(0,-1){12}}

    \put(74,76){\line(1,0){5}}
    \put(79,76){\vector(0,-1){12}}

    \put(0,50){\framebox(36,14)[c]{\scriptsize \textbf{\shortstack{Problem PI \\ ($\min$)}}}}
    \put(64,50){\framebox(36,14)[c]{\scriptsize \textbf{\shortstack{Problem PII\\($\max$)}}}}

    \put(15,50){\vector(0,-1){7}}
    \put(79,50){\vector(0,-1){7}}

    \put(10,29){\framebox(74,14)[c]{\scriptsize \textbf{\shortstack{Metric Gap\\$\Delta_m$}}}}

    \put(17.5,29){\vector(0,-1){7}}
    \put(76.5,29){\vector(0,-1){7}}

    \put(0,0){\framebox(56,22)[c]{\scriptsize \textbf{\shortstack{Scenario 1:\\$\Delta_m \approx 0$\\[1.5mm]\tiny (Stable systems)}}}}
    \put(59,0){\framebox(56,22)[c]{\scriptsize \textbf{\shortstack{Scenario 2:\\$\Delta_m \gg 0$\\[1.5mm]\tiny (Stratified populations)}}}}
\end{picture}
\begin{center}
    \vspace{-1mm}
    \fontsize{7}{9}\selectfont
    \textbf{Fig.~\thesection.1} Methodological  diagram.
\end{center}
\end{minipage}
\vspace{2mm}

    The following theorem \cite{KABaChu26}  establishes the necessary condition  for  a solution to an extremal problem.
\begin{theorem}\label{T1} There exists a solution of the problem  PI (PII), moreover if $X$ is a solution   then
 \begin{equation}\label{5}
\exists\Lambda=\mathrm{diag}\{\lambda_{1},\lambda_{2},\dots,\lambda_{m}\},\,\lambda_{j}\in \mathbb{R}:\;     H X- T_{\zeta}X\Lambda=0.
\end{equation}
\end{theorem}
  Thus, we are motivated   to study   solutions to the stationary equation  \eqref{5} under the  restriction given by the connection condition (for more details, see \cite{KABaChu26}). Let us  rewrite the stationary system \eqref{5} in the following form
 \begin{equation}\label{6}
 H\mathbf{x} _{ j } = \lambda_{j}T_{\zeta}\mathbf{x} _{ j },\;(T_{\zeta}\mathbf{x} _{ j },\mathbf{x} _{ k })_{\mathbb{C}^{n}}=\delta_{jk},\;j,k=1,2,\dots,m.
 \end{equation}
 Motivated by the formula $H(X)= |\psi(X)|\{ e^{-i\theta}L+e^{i\theta}L^{\ast}\},\;\theta:= \mathrm{arg}\, \psi(X),$ we have come to the following generalized eigenvalue problem
\begin{equation}\label{7}
\Psi(\theta)\mathbf{e}_{ j } =\mu_{j} T_{\zeta}\mathbf{e}_{ j },\;(T_{\zeta}\mathbf{e} _{ j },\mathbf{e} _{ k })_{\mathbb{C}^{n}}=\delta_{jk},\;j,k=1,2,\dots,\xi,
\end{equation}
where $\Psi(\theta)=e^{-i\theta}L+e^{i\theta}L^{\ast},\;\theta\in \mathbb{R},$    the eigenvalues are arranged in increasing order, $\xi$ denotes the number of non-zero eigenvalues.   Since the operator $\Psi(\theta)$ depends upon the parameter $\theta$ only, it follows that    $\mu_{j}:=\mu_{j}(\theta),\; \mathbf{e}_{ j }:=\mathbf{e}_{ j }(\theta).$   The generalized  eigenvalue problem  \eqref{7}  is equivalent to the eigenvalue problem
\begin{equation}\label{8}
T(\theta)\mathbf{v}_{j} =\mu_{j}\mathbf{v}_{j},\;(\mathbf{v}_{j},\mathbf{v}_{k})_{\mathbb{C}^{n}}=\delta_{jk},\,\mathbf{v}_{j}=\sqrt{T_{\zeta}}\mathbf{e}_{j},\;j,k=1,2,\dots,\xi,
\end{equation}
where $T(\theta):=T^{-1/2}_{\zeta}\Psi(\theta)T^{-1/2}_{\zeta}.$ Let us show that $\xi=n-1.$  In accordance with the orthogonal decomposition theorem for a selfadjoint operator, we have $\mathrm{dim} \,\mathbb{C}^{n}=\mathrm{dim}\, \mathrm{R}(T(\theta))+\mathrm{dim}\, \mathrm{N}(T(\theta)).$
 Applying Lemma \ref{L1} and taking into account the properties of the operator $T^{-1/2}_{\zeta},$  we can easily establish  the fact   $\mathrm{dim}\, \mathrm{N}(T(\theta))=1.$ Since the system of  eigenvectors of the operator $T(\theta)$ is complete in $\mathrm{R}(T(\theta)),$ we obtain the desired result. To establish the correspondence between formulas \eqref{7} and \eqref{6}, we should note that     Lemma \ref{L1} and  relation \eqref{1}   provide  the implication $\mathrm{tr}\{X^{\ast}LX\}=0 \Rightarrow \mathbf{x}_{j}\in \mathrm{N}(L)= \mathrm{N}(\Psi(\theta)),\,j=1,2,\ldots,m.$

In accordance with \cite{KABaChu26}, we have
$
\mathrm{arg} \,\{\psi(X)\}:\mathfrak{M}_{m}^{n}\rightarrow \mathbb{J} \subset (0,\pi/2),
$
where $\mathbb{J}$ is a closed interval.  Further, we can assume that $\theta\in \mathbb{J}.$ Define the matrices
$
E_{p}(\theta):=\left(\mathbf{e}_{ p_{1} }(\theta),\mathbf{e}_{ p_{2} }(\theta),\ldots,\mathbf{e}_{ p_{m} }(\theta)\right)\in \mathbb{C}^{n\times m},\;p_{j}\in \{1,2,\dots,n-1\},\;p_j\neq p_q,\;j\neq q.
$
It is clear that we have $C_{n-1}^{m}$ matrices formed from the various $m$-element  subsets of the eigenvectors, i.e., we can enumerate the indices of the groups as follows     $p=1,2,\dots,C_{n-1}^{m},$ where we put
 $E_{1}(\theta):=(\mathbf{e}_{ 1 }(\theta),\mathbf{e}_{ 2 }(\theta),\dots,\mathbf{e}_{ m }(\theta)),\,E_{2}(\theta):=(\mathbf{e}_{ n-m  }(\theta),\mathbf{e}_{ n-m+1  }(\theta),\dots,\mathbf{e}_{ n-1 }(\theta)).$
Here, we assume that  the   matrices corresponding to other values of the index $p$  are formed from  the columns  of eigenvectors without any special structure. Note that to satisfy the stationary equation \eqref{5} it is not sufficient for $E_p(\theta)$ to be formed from a set of eigenvectors corresponding to an index $p.$     The following additional condition is necessary $\mathrm{arg}\,\psi(E_{p}(\theta))=\theta.$ Throughout the paper, we consider the matrices $E_{p}(\theta)$ formed from the  columns of eigenvectors of the problem \eqref{7} up to a phase multiplier.

Using terminology \cite[p.63]{firstab_lit:kato1980}, consider an operator-valued function
$$
T(\kappa):=T^{-1/2}_{\zeta}\Psi(\kappa)T^{-1/2}_{\zeta},\;\kappa\in \mathbb{C}.
$$
This   operator-valued function is holomorphic on $\mathbb{C}$ and has a decomposition into a Taylor series which can be written in terms of operators (1.2)  \cite[p.63]{firstab_lit:kato1980} as follows
$$
T(\kappa)=T+\kappa T^{(1)}+\kappa^{2} T^{(2)}+\ldots .
$$
Note that the operators $T, T^{(1)},T^{(2)},\ldots$ are symmetric, and the sum  $T(\kappa)$ is also symmetric for real $\kappa,$ i.e.,  we have $T(\kappa)^{\ast}=T(\kappa).$ Therefore, in accordance with the terminology of Chapter II, \S 6.1 \cite[p.120]{firstab_lit:kato1980} the family $\{T(\kappa)\}$  is said to be symmetric.

Consider a Riesz projector (eigenprojection)
$$
  P_{j} (\kappa)=-\frac{1}{2\pi i}\oint\limits_{\Gamma}( T(\kappa)-\lambda I)^{-1}d\lambda,\;\kappa\in \mathbb{C},
$$
where  $\Gamma\subset \mathbb{C}$ is a closed contour  belonging to the resolvent subset of the operator $T(\kappa),$  bounding a domain containing the  eigenvalue $\mu_{j}(\kappa )$ only (see \cite{Kuk2025ess} for more details).

The following lemma improves a more general result established in the paper \cite{KABaChu26}.

\begin{lemma}\label{L2}
The matrix functions $E_{p}(\theta)$ are analytic on the real axis. Moreover, there exists a unique  solution to the following system
\begin{equation}\label{9}
\mathrm{arg}\, \psi\{E_{p}(\theta)\}=\theta,\;\theta\in \mathbb{J},\;p=1,2,\dots,C_{n-1}^{m}.
\end{equation}
\end{lemma}
\begin{proof}
In accordance with Theorem 6.1 \cite[p.120]{firstab_lit:kato1980}, since the   family  of holomorphic  operators $T (\kappa)$ is symmetric,  the eigenvalues
$\mu_{j} (\theta)$ and the eigenprojections $P_{j} (\theta),\,\theta\in \mathbb{R}$ are analytic.
Therefore, using  results established in  Chapter II, \S 4.5, \S 6.2 \cite[p.121]{firstab_lit:kato1980}, we can
  construct an analytic (holomorphic in accordance with \cite{firstab_lit:kato1980}) transformation function $U (\theta)$  satisfying the   condition $U(\theta)P_{j} (\theta_{0})U^{-1}(\theta)=P_{j} (\theta).$  Furthermore, $U (\theta),\,\theta\in \mathbb{R}$ is unitary.  In accordance with the reasoning  given  in Chapter II, \S 6.2 \cite[p.122]{firstab_lit:kato1980}, it follows that if we apply    the transformation function $U (\theta)$  to   orthonormal  eigenvectors corresponding to the eigenvalue $\mu_{j}(\theta_{0} )\in \mathbb{R},$ we obtain    the orthonormal  eigenvectors corresponding to the eigenvalue $\mu_{j}(\theta).$ Since the operator $T(\theta)$ is selfadjoint, the eigenvectors corresponding to distinct eigenvalues are orthogonal. Therefore, the basis
  $\mathbf{v}_{j}(\theta)=U(\theta)\mathbf{v}_{j}(\theta_{0}),\,j=1,2,\ldots,n-1$ is orthonormal in $\mathrm{R}(T (\theta)).$ Note that since the eigenprojections $P_{j} (\theta),\,\theta\in \mathbb{R}$ are analytic, the eigenvectors    $\mathbf{v}_{j}(\theta)$ corresponding to the  problem \eqref{8} are analytic. Thus, we can assume that  $T_{\zeta}$-orthonormal generalized  eigenvectors corresponding to the  problem  \eqref{7} are analytic on the real axis and consequently continuous on $\mathbb{J}.$

Consider a function
$
G_{p}(\theta)=\mathrm{arg}\, \mathrm{tr}\{E^{\ast}_{p}(\theta)LE_{p}(\theta)\},\,\theta\in \mathbb{J}.
$
In accordance with the above,  the matrix functions $E_{p}(\theta)$ are continuous on the set $\mathbb{J}\subset (0,\pi/2)$  with respect to an arbitrary norm due to the norm  equivalence property corresponding to  the complex Euclidean space.   It is clear that  $E_{p}(\theta)\in \mathfrak{M}_{m}^{n}.$ In accordance with  \cite{KABaChu26}
the mapping $\mathrm{arg} \,\{\psi(X)\}:\mathfrak{M}_{m}^{n}\rightarrow \mathbb{J}$ is continuous. Thus, the mapping $G_{p}(\theta):\mathbb{J} \rightarrow \mathbb{J}$ is    continuous.   Applying the Brouwer fixed-point theorem, which states that   {\it  every continuous map of the closed unit
 ball in the Euclidean space into itself has a fixed point},  we conclude that   there exists a fixed point $\theta_{p}\in \mathbb{J}$ of  the equation \eqref{9}.
      Let us show that the  fixed point $\theta_{p}\in \mathbb{J}$ is unique. In accordance with the above,  the function
$$
\rho_{p}(\theta) =\sum\limits_{j=1}^{m}\mathbf{e}^{\ast}_{ p_{j} }(\theta)\Psi(\theta)\mathbf{e}_{ p_{j} }(\theta),
$$
  is analytic on the real axis.   Let us show that $\rho'_{p}(\theta_{p})=0.$
Using Lemma 2  \cite{KABaChu26}, we get
$$
\rho_{p}'(\theta)= \sum\limits_{j=1}^{m}\mathbf{e}^{\ast}_{p_{j}}\!(\theta)\Psi'(\theta)\mathbf{e} _{p_{j}}(\theta)=
\sum\limits_{j=1}^{m}\mathbf{e}^{\ast}_{p_{j}}\!(\theta)i\left(e^{i\theta}L^{\ast}-e^{-i\theta}L \right)  \mathbf{e} _{p_{j}}(\theta)=
$$
$$
=ie^{i\theta}\mathrm{tr}\{E^{\ast}_{p}(\theta)L^{\ast}E_{p}(\theta)\}-ie^{-i\theta}\mathrm{tr}\{E^{\ast}_{p}(\theta)LE_{p}(\theta)\}.
$$
Therefore
$$
\rho_{p}'(\theta_{p})=ie^{i\theta_{p}}\bar{\psi} (E_{p}(\theta_{p})) -i e^{-i\theta_{p}} \psi  (E_{p}(\theta_{p}))=-2\mathrm{Re}\left\{ie^{-i\theta_{p}} \psi  (E_{p}(\theta_{p})) \right\}=
$$
$$
=2\mathrm{Im}\left\{ e^{-i\theta_{p}} \psi  (E_{p}(\theta_{p})) \right\} =2\mathrm{Im}\left\{ e^{-i\theta_{p}}e^{ i\theta_{p}} |\psi  (E_{p}(\theta_{p}))| \right\}=0.
$$
Now if we assume that there exists a fixed point $\tilde{\theta}_{p}\neq \theta_{p},$ then using the same reasoning  we conclude that $\rho_{p}'(\tilde{\theta}_{p})=0.$ However, the latter contradicts   Lemma 3 \cite{KABaChu26}, which states  that  $\rho''_{p}(\theta)<0,\;\theta\in (a,b),$  where  $(a,b)\subset \mathbb{R}$ is an arbitrary interval containing the points $ \theta_{p},\,\tilde{\theta}_{p}.$   Thus we conclude that the fixed point $\theta_{p}$ is unique. The proof is complete.
\end{proof}

\subsection{The  gradient  ascent method}
  Using relation \eqref{1},  we can reformulate the  problem PII in   terms of the given construction. The latter reveals the sense of the maximization operation described in detail  in the introduction section.

 In the paper \cite{KABaChu26}, we establish the fact that the maximum problem has a solution due to the exclusion method. However,   fundamental reasonings  justifying the formal logical deduction should be presented in an expanded form. Motivated by this observation, we propose a variant of the gradient ascent method in order to present a rigorous proof of the formally established   fact.

\begin{lemma}\label{L3} The following relation holds
\begin{equation}\label{10}
 \max\limits_{X^{\ast}T_{\zeta}X= I_{m}} \sum\limits_{j=1}^{m}(\Psi(\theta)\mathbf{x}_{j},\mathbf{x}_{j})_{\mathbb{C}^{n}}=\sum\limits_{s=n-m }^{n-1}\mu_{s}(\theta),\;\theta\in (0,\pi/2).
\end{equation}
\end{lemma}
\begin{proof} For simplicity of notation,  we assume that the eigenvalues  are arranged  in decreasing order of their absolute values.
 Applying Lemma \ref{L2}, we can choose  $T_{\zeta}$-orthonormal analytic eigenvectors $\{\mathbf{e}_{j}\}^{\xi}_{1}.$
Using the property of the selfadjoint operator $T(\theta)$ and taking into account the equivalence between the eigenvalue problems \eqref{7} and  \eqref{8}, we can easily prove
$$
\mathbb{C}^{n}=\mathrm{span}\{\mathbf{1}\}\oplus_{\,T_{\zeta}}\mathrm{span}\{\mathbf{e}_{1},\mathbf{e}_{2},\ldots,\mathbf{e}_{n-1}\}.
$$
Thus,  the following representation holds
$$
 \sum\limits_{j=1}^{m}( \Psi(\theta) \mathbf{x} _{j}  ,     \mathbf{x} _{j}   )_{\mathbb{C}^{n}}=  \sum\limits_{j=1}^{m}\left( \sum\limits_{s=1}^{\xi}c_{sj } \mu_{s}
          T_{\zeta}\mathbf{e}_{s},\sum\limits_{s=1}^{\xi} c_{sj } \mathbf{e}_{s} \right)_{\! \mathbb{C}^{n}}=
          \sum\limits_{s=1}^{\xi}\mu_{s} \sum\limits_{j=1}^{m}  |c_{sj }|^{2},
$$
where we have used the decomposition
$$
\mathbf{x} _{j}= \sum\limits_{s=1}^{n-1} c_{sj }\mathbf{e}_{s}+c_{nj} \mathbf{1},
$$
the values $\mu_{s},\,c_{sj },$ and the vectors $\mathbf{e}_{s}$ depend on the parameter $\theta,$ according to the   agreement, we have $\mu_{1}>\mu_{2}>...>\mu_{\xi}.$ It is clear that the eigenvalues of the problem  \eqref{7} are positive numbers. Using the connection condition \eqref{2}, we get
\begin{equation*}
 \delta_{jk}=(  T_{\zeta} \mathbf{x} _{j} ,\mathbf{x} _{k}   )_{\mathbb{C}^{n}}=
 \left(\sum\limits_{s=1}^{n}c_{sj }T_{\zeta}\mathbf{e}_{s},\sum\limits_{s=1}^{n} c_{sk }\mathbf{e}_{s} \right)_{\!\!\!\mathbb{C}^{n}}=
 \sum\limits_{s=1}^{\xi} c _{sj }\overline{c} _{sk }+c_{nj} \bar{c}_{nk}\,\mathrm{tr}(T_{\zeta}),
\end{equation*}
where $\mathbf{e}_{n}=\mathbf{1}.$
Applying this formula, rewriting the left-hand side of \eqref{10}, we come to the related  problem
$$
\max\limits \sum\limits_{s=1}^{\xi}\mu _{s} \sum\limits_{j=1}^{m}  |c_{sj }|^{2},
\;\;\sum\limits_{j=1}^{m}\sum\limits_{s=1}^{\xi} |c _{sj }|^{2}+\mathrm{tr}(T_{\zeta})\sum\limits_{j=1}^{m}|c _{nj }|^{2}=m,
$$
Denote
$$
u_{0}:=\mathrm{tr}(T_{\zeta})\sum\limits_{j=1}^{m}|c _{nj }|^{2},\;u_{s}:=\sum\limits_{j=1}^{m}  |c_{sj }|^{2},\;s=1,2,\ldots,\xi.
$$
Then  the expression corresponding to the  main problem can be rewritten in the form
$$
d_{\xi}(\mathbf{u}) :=\sum\limits_{s=1}^{\xi}\mu _{s}u_{s}  ,\;\;\sum\limits_{s=1}^{\xi}u_{s}+u_{0}=m,\;\mathbf{u}=(u_{0},u_{1},\ldots,u_{\xi}).
$$
Substituting, we get
$$
h_{\xi-1}(\mathbf{u}):=\sum\limits_{s=1}^{\xi-1}u_{s}\mu _{s}+\mu _{\xi}\left(m-\sum\limits_{s=1}^{\xi-1}u_{s}-u_{0}\right)=
\sum\limits_{s=1}^{\xi-1}u_{s} \left(\mu _{s}-\mu_{\xi}\right)-u_{0}\mu _{\xi}+m \mu _{\xi},
$$
$$
\sum\limits_{s=0}^{\xi-1}u_{s}\leq m.
$$
Therefore, the component of the gradient $\nabla h_{\xi-1}$ corresponding to   $u_{0}$ is   negative, whereas the  other  coordinates are positive  $ \mu_s - \mu_{\xi} > 0.$ Combining this fact with the geometric properties of the hyperplane, we immediately conclude that at the maximum point the non-negative parameter $u_{0}$ must take  its lower bound, i.e., $u_{0} = 0.$ The latter eliminates the kernel contribution. Thus, we  have
$$
\max \limits_{\mathbf{u}\in \mathbb{R}^{\xi-1}_{m}}  h_{\xi-1}(\mathbf{u})= \max\limits_{\mathbf{u}\in \mathbb{T}^{\xi-1}_{m}}\sum\limits_{s=1}^{\xi-1}u_{s}\mu_{s}=  \max\limits_{\mathbf{u}\in \mathbb{T}^{\xi-1}_{m}}  d_{\xi-1}(\mathbf{u}),
$$
where
$$
\mathbb{R}^{n}_{m}:=\left\{\mathbf{x}\in \mathbb{R}^{n}:\, \sum\limits_{s=1}^{n}x_{s}\leq m \right\},\;\mathbb{T}^{n}_{m}:=\left\{\mathbf{x}\in \mathbb{R}^{n}:\, \sum\limits_{s=1}^{n}x_{s}= m \right\},\;n,m\in \mathbb{N}.
$$
Apparently, in finding the maximum point, we have some freedom in choosing the coefficients, thus at this step of the proof, we   assume    $c_{\xi j}=0,\,j=1,2,\ldots,m.$
Following  the same reasoning, we get
\begin{equation*}
h_{\xi-2}(\mathbf{u}):=
\sum\limits_{s=1}^{\xi-2}u_{s} \left(\mu _{s}-\mu _{\xi-1}\right)+m \mu _{\xi-1},\;\sum\limits_{s=1}^{\xi-2}u_{s}\leq m.
\end{equation*}
Therefore, the constant vector $\nabla h_{\xi-2}$ has positive coordinates. Analogously to the above, we get
$$
\max\limits_{\mathbf{u}\in \mathbb{R}^{\xi-2}_{m}}  h_{\xi-2}(\mathbf{u})=   \max\limits_{\mathbf{u}\in \mathbb{T}^{\xi-2}_{m}}  d_{\xi-2}(\mathbf{u}),
$$
where the latter relation is obtained due to the choice  $c_{\xi-1j}=c_{\xi j}=0,\,j=1,2,\ldots,m.$
Following the same reasoning, we come to the problem
$$
\max\limits_{X^{\ast}T_{\zeta}X= I} \sum\limits_{j=1}^{m}(\Psi(\theta)\mathbf{x}_{j},\mathbf{x}_{j})_{\mathbb{C}^{n}}=\max\limits_{ C ^{\ast} C =I_{m}} \sum\limits_{s=1}^{m}\mu _{s} \sum\limits_{j=1}^{m}  |c_{sj }|^{2},
$$
where we assume  $c_{m+1j}=c_{m+2j}=\ldots=c_{\xi j}=0,\,j=1,2,\ldots,m.$
Therefore, we can restrict the problem considering the   relation
\begin{equation}\label{11}
      \sum\limits_{s=1}^{m}\mu_{s} \sum\limits_{j=1}^{m}  |c_{sj }|^{2}, \;
\sum\limits_{s=1}^{m} c _{sj }\overline{c }_{sk }  =\delta_{kj},\,j,k=1,2,\ldots,m.
\end{equation}
 In  matrix form, we have
$
 C ^{\ast} C =I_{m},\;  \Rightarrow  C C ^{\ast}=I_{m},\;C:=\{c_{sj}\}\in \mathbb{C}^{m\times m},
$
i.e.,
$$
\sum\limits_{j=1}^{m} c _{sj  }\overline{c }_{kj } =\delta_{sk},\,s,k=1,2,\ldots,m.
$$
Substituting   into \eqref{11} and    arranging   the eigenvalues in increasing order of their absolute values, we obtain \eqref{10}. The proof is complete.
\end{proof}

\begin{theorem} There exist  solutions to the   problems PI and  PII represented by the matrices $E_{1}(\theta_{1})$ and $E_{2}(\theta_{2})$ respectively. Moreover, both the corresponding minimum point and the maximum point are unique.
\end{theorem}
\begin{proof}
In accordance with \cite{KABaChu26}, the matrix $E_{1}(\theta_{1})$ corresponds to a local minimum point.   Applying Theorem \ref{T1}, Lemma \ref{L2}, we conclude that the local minimum is unique. Therefore, $E_{1}(\theta_{1})$ is the  solution to the  problem PI.  Consider the  problem PII. Note that
$
\mathrm{tr}\{X^{\ast}HX\}=\overline{\psi(X)}\psi(X) + \psi(X)\overline{\psi(X)}=2|\psi(X)|^{2},
$
therefore
$$
|\psi(X)|^{2}=
\frac{1}{2}\sum\limits_{j=1}^{m}(H\mathbf{x}_{j},\mathbf{x}_{j})_{\mathbb{C}^{n}}=
\frac{1}{2}|\psi(X)| \sum\limits_{j=1}^{m}(\Psi(\theta)\mathbf{x}_{j},\mathbf{x}_{j})_{\mathbb{C}^{n}},\;\theta=\mathrm{arg}\,\psi(X).
$$
Applying Lemma \ref{L3},   using the fact $\rho'_{2}(\theta_{2})=0,\,\rho''_{2}(\theta_{2})<0,$ considered in detail  in Lemma \ref{L2}, we conclude that  for an arbitrary   $X\in \mathbb{C}^{n\times m}$ satisfying the connection condition  $X^{\ast}T_{\zeta}X= I_{m},$ the following inequality holds
$$
|\psi(X)|\leq \frac{1}{2}\sum\limits_{s=n-m }^{n-1}\mu_{s}(\theta)\leq \frac{1}{2}\sum\limits_{s=n-m }^{n-1}\mu_{s}(\theta_{2})  =|\psi(E_{2}(\theta_{2}))|,\,\theta=\mathrm{arg}\,\psi(X).
$$
The latter relation proves that $E_{2}(\theta_{2})$ is a solution to the  problem PII. In accordance with  Subsection 3.2 \cite{KABaChu26} a local maximum point is represented by a matrix  $E_{2}(\theta).$  Applying   Lemma \ref{L2}, we conclude that the local maximum point  is unique.
\end{proof}

\section{Conclusions}

Having  analyzed  the main mathematical principles forming  the concept of  unsupervised topological alignment, we present a natural algebraic structure coupling the heterogeneous  datasets as well as their images. In this regard, the elegant mathematical   generalization of the graph Laplacian  method was obtained over the field of   complex numbers.

The minimum problem serves  as a  macroscope totally smoothing out individual  fluctuations and anomalies to preserve global continuity and reveal the invariant structure of the macro-biological background. However, this total flattening inevitably suffers from an oversmoothing effect,
eliminating critical micro-variations within highly localized subsets of objects.

Conversely,  the maximum problem  serves as a high-contrast  microscope  that explicates hidden sample stratification and uncovers latent cryptic components. Rather than smoothing the geometric structure, it  identifies directions of extreme trigonometric tension   generated by phase shifts between the real and imaginary parts of the coupled graph Laplacian operator.

In order to  evaluate the underlying structural conflict between these opposite geometric tendencies on the compact manifold, we have introduced a special indicator defined as the norm of the difference between the optimal embedding matrices. This indicator measures the maximum geometric divergence between the created  linear subspaces. The resulting theoretical approach is  fundamentally novel, establishing a strict mathematical basis to quantify structural desynchronization and providing relevant biological applications.\\

\noindent {\bf Funding information}\\ This study was funded by State Assignment of Russian National Research Medical University No. 125022602916-4.

\end{document}